%% file: main_short_v2.tex
\documentclass{article}

\usepackage{microtype}
\usepackage{graphicx}
\usepackage{booktabs}

\usepackage{hyperref}

\usepackage[accepted]{icml2026}

\usepackage{amsfonts,amsmath,amssymb,amsthm}
\usepackage{enumitem}
\usepackage{algorithm}
\usepackage{algorithmic}
\usepackage{xcolor}
\usepackage{mdframed}

\newmdenv[
  linecolor=blue!60!cyan,
  backgroundcolor=blue!5,
  linewidth=0.5pt,
  roundcorner=2pt,
  skipabove=4pt,
  skipbelow=4pt,
  innerleftmargin=6pt,
  innerrightmargin=6pt,
  innertopmargin=6pt,
  innerbottommargin=6pt
]{resultbox}

\input{math_commands.tex}

\theoremstyle{plain}
\newtheorem{theorem}{Theorem}[section]
\newtheorem{proposition}[theorem]{Proposition}

\theoremstyle{definition}

\theoremstyle{remark}

\def\R{\mathbb{R}}  
\def\E{\mathbb{E}}

\def\<{\langle} \def\>{\rangle}

\icmltitlerunning{Generative Modeling via Kernelized Stochastic Interpolants}

\begin{document}

\twocolumn[
  \icmltitle{Generative Modeling via Kernelized Stochastic Interpolants}

  \begin{icmlauthorlist}
    \icmlauthor{Florentin Coeurdoux}{cfm}
    \icmlauthor{Etienne Lempereur}{ens}
    \icmlauthor{Nathana\"el Cuvelle--Magar}{ens}
    \icmlauthor{St\'ephane Mallat}{cdf,fi}
    \icmlauthor{Eric Vanden-Eijnden}{courant,cfm}
  \end{icmlauthorlist}

  \icmlaffiliation{cfm}{Capital Fund Management, Paris, France}
  \icmlaffiliation{ens}{D\'epartement d'informatique, ENS, Universit\'e PSL, Paris, France}
  \icmlaffiliation{cdf}{Coll\`ege de France, Paris, France}
  \icmlaffiliation{fi}{CCM, Flatiron Institute, New York, USA}
  \icmlaffiliation{courant}{Courant Institute of Mathematical Sciences, New York University, New York, USA}

  \icmlcorrespondingauthor{Florentin Coeurdoux}{coeurdoux.florentin@gmail.com}

  \icmlkeywords{generative modeling, stochastic interpolants, kernel methods, diffusion models, score-based models}

  \vskip 0.3in
]

\printAffiliationsAndNotice{}

\begin{abstract}
We develop a kernel method for generative modeling within the stochastic interpolant framework, replacing neural network training with linear systems. The drift of the generative SDE is $\hat b_t(x) = \nabla\phi(x)^\top\eta_t$, where $\eta_t\in\R^P$ solves a $P\times P$ system computable from data, with $P$ independent of the data dimension $d$. Since estimates are inexact, the diffusion coefficient $D_t$ affects sample quality; the optimal $D_t^*$ from Girsanov diverges at $t=0$, but this poses no difficulty and we develop an integrator that handles it seamlessly. The framework accommodates diverse feature maps---scattering transforms, pretrained generative models etc.---enabling generation and model combination without neural network training. We demonstrate the approach on financial time series, turbulence, and image generation.
\end{abstract}

\section{Introduction}
\label{sec:intro}

Generative diffusion models based on dynamical transport of measure---including score-based models~\citep{song2020score}, flow matching~\citep{lipman2022flow}, rectified flows~\citep{liu2022flow}, and stochastic interpolants~\citep{albergo2022building,albergo2023stochastic}---require learning a time-dependent drift or score function, typically by training a neural network on a regression objective. This training step is computationally expensive.

Here we show that this neural network training can be replaced by a \emph{kernel method}. Given a feature map $\phi:\R^d\to\R^P$, the drift of the generative process can be approximated as $\hat b_t(x) = \nabla\phi(x)^\top\eta_t$, where $\eta_t\in\R^P$ solves a $P\times P$ linear system computable from data. The score estimate follows from the exact drift-score relation. This reduces the learning problem from neural network training to solving linear systems---one per time step, precomputed before generation.

Since the feature map is finite-dimensional, the drift and score estimates are approximate. The diffusion coefficient $D_t$ in the generative SDE, which would be irrelevant with exact estimates, now affects sample quality. We use the optimal schedule $D_t^*$~\citep{ma_sit_2024,chen2024probabilistic}, obtained by minimizing a path KL bound via Girsanov's theorem. The resulting $D_t^*$ is large near $t=0$ (strong diffusion when the distribution is nearly Gaussian) and vanishes near $t=1$ (pure transport near the data). We develop an integrator that handles both limits seamlessly.

The framework accommodates diverse feature maps, like scattering transforms for time series and random fields and pretrained generative models. The latter is particularly notable: pretrained velocity fields can be used directly as feature gradients, enabling combination of models with different architectures or training stages through linear algebra, without any retraining.
This work is complementary to Moment-Guided Diffusion~\citep{lempereur2025mgd}, which uses similar feature maps but constrains moments rather than regressing the drift.

\paragraph{Main contributions.}
\begin{itemize}[leftmargin=0.2in,itemsep=1pt]
    \item A kernel reformulation of stochastic interpolants where the drift is estimated by solving a $P\times P$ linear system, with $P$ the number of features independent of the input dimension $d$ (Proposition~\ref{th:main}).
    \item Application of the optimal diffusion coefficient $D_t^*$~\citep{ma_sit_2024,chen2024probabilistic} to the kernel setting, with a KL bound on generation error (Proposition~\ref{prop:optimal:eps}): $D_t^*$ diverges at $t=0$ (Gaussian regime) and vanishes at $t=1$ (ODE transport).
    \item An integrator for the optimal SDE that handles $D_0^*=\infty$ without clamping (Section~\ref{sec:integrator}).
    \item Practical constructions using scattering transforms and pretrained generative models, with the ability to combine them (Section~\ref{sec:practical}).
\end{itemize}

\paragraph{Related work.}
Our approach builds on stochastic interpolants~\citep{albergo2022building,albergo2023stochastic}, reformulating them as kernel methods~\citep{muandet2017kernel}. The optimal diffusion coefficient $D_t^*$ we use was derived in~\citet{ma_sit_2024,chen2024probabilistic}; a similar expression appears in the stochastic optimal control setting~\citep{havens2025adjoint}. Closely related is Moment-Guided Diffusion (MGD)~\citep{lempereur2025mgd}, which also uses feature maps within stochastic interpolants but imposes moment constraints $\E[\phi(X_t)] = \E[\phi(I_t)]$ to sample maximum entropy distributions. MGD requires $D_t\to\infty$ to reach the maximum entropy limit, whereas our approach uses the finite optimal $D_t^*$ to minimize generation error. Our method is simpler---solving a linear regression beforehand rather than enforcing constraints throughout the dynamics---but offers less control over which statistics are matched.

RKHS formulations of diffusion models~\citep{maurais2024samplingunittimekernel,yi2024denoisingdatameasurementerror} use standard kernels rather than problem-specific feature maps; MMD-based methods~\citep{gretton2012kernel,galashov2024deepmmdgradientflow,debortoli2025distributionaldiffusionmodelsscoring} require engineering feature maps, whereas we use $\nabla\phi$ directly from existing models or transforms. For model combination, our approach requires no retraining and operates on velocity fields, unlike parameter averaging~\citep{wortsman2022model,biggs2024diffusion} or LoRA merging~\citep{hu2021lora,shah2024ziplora}.

\section{Theory}
\label{sec:theory}

\subsection{Stochastic Interpolants}
\label{sec:interpolants}

Let $\mu$ be a target probability distribution on $\R^d$, known only through samples, and assume that $\mu$ has a differentiable, everywhere positive density. The \textbf{stochastic interpolant} between noise $z\sim\mathsf{N}(0,\mathrm{Id}_d)$ and data $a\sim\mu$ is
\begin{equation}
    \label{eq:interpolant}
    I_t = \alpha_t z + \beta_t a, \quad z\sim\mathsf{N}(0,\mathrm{Id}_d),\; a\sim\mu,\; a\perp z,\; t\in[0,1],
\end{equation}
where $\alpha,\beta\in C^1([0,1])$ with $\alpha_0=\beta_1=1$, $\alpha_1=\beta_0=0$, and $\dot\alpha_t<0$, $\dot\beta_t>0$ on $t\in(0,1)$ (e.g., linear $\alpha_t=1-t$, $\beta_t=t$ or trigonometric $\alpha_t = \cos(\tfrac12 \pi t)$, $\beta_t = \sin(\tfrac12 \pi t)$).
From~\citet{albergo2023stochastic}, $I_t$ has density $\rho_t$, which coincides with that of the solution $X_t$ of
\begin{equation}
    \label{eq:sde}
    \begin{gathered}
        dX_t = b_t(X_t)\,dt + D_t\, s_t(X_t)\,dt + \sqrt{2D_t}\,dW_t, \\
        X_0 = z\sim\mathsf{N}(0,\mathrm{Id}_d).
    \end{gathered}
\end{equation}
where $D_t \ge 0$ is a tunable diffusion coefficient, the velocity field is $b_t(x) = \E[\dot I_t | I_t = x]$ for $t\in[0,1]$, and the score $s_t(x) = \nabla\log\rho_t(x)$ satisfies, by Tweedie's identity, $s_t(x) = -\alpha_t^{-1}\E[z|I_t=x]$ for $t\in[0,1)$. The velocity is the minimizer of the regression loss
\begin{equation}
    \label{eq:velocity:loss}
    L_b[\hat b_t] = \E\big[|\hat b_t(I_t) - \dot I_t|^2\big],
\end{equation}
and $b_t(x)$ and $s_t(x)$ are related by
\begin{equation}
    \label{eq:score:drift}
    s_t(x) = \frac{\beta_t\, b_t(x) - \dot\beta_t\, x}{\alpha_t\gamma_t}\,,
    \qquad \gamma_t := \alpha_t\dot\beta_t - \dot\alpha_t\beta_t > 0.
\end{equation}
Since $I_1 = a \sim \mu$ by construction, integrating~\eqref{eq:sde} yields samples from the target: $X_1 \sim \mu$.

\subsection{Drift Estimation via Linear Systems}
\label{sec:linear}

Given a \textbf{feature map} $\phi:\R^d\to\R^P$ with Jacobian $\nabla\phi(x)\in\R^{P\times d}$ (see Section~\ref{sec:practical} for concrete choices), we approximate the velocity as $\hat b_t(x) = \nabla\phi(x)^\top\eta_t$ with $\eta_t\in\R^P$, reducing~\eqref{eq:velocity:loss} to a quadratic problem in $\eta_t$.
\begin{proposition}[Drift estimation]
    \label{th:main}
    Define the Gram matrix of the feature gradients $\{\nabla\phi_i\}_{i=1}^P$ under the law of $I_t$:
    \begin{equation}
        \label{eq:Kt}
        K_t = \E\big[\nabla\phi(I_t)\cdot\nabla\phi(I_t)^\top\big] \in \R^{P\times P}.
    \end{equation}
    Assume $K_t$ is positive-definite. Then the minimizer of $L_b$ over $\hat b_t(x) = \nabla\phi(x)^\top\eta_t$ is
    \begin{equation}
        \label{eq:bt:kernel}
        \hat b_t(x) = \nabla\phi(x)^\top\eta_t, \qquad
        K_t\,\eta_t = \E\big[\nabla\phi(I_t)\cdot\dot I_t\big].
    \end{equation}
\end{proposition}

\begin{proof}
    Expanding $L_b = \eta_t^\top K_t \eta_t - 2\,\E[\nabla\phi(I_t)\cdot\dot I_t]^\top\eta_t + \E[|\dot I_t|^2]$, the unique minimizer under positive-definite $K_t$ satisfies $K_t\eta_t = \E[\nabla\phi(I_t)\cdot\dot I_t]$.
\end{proof}

Given samples $(z_n,a_n)_{n=1}^N$ with $z_n\sim\mathsf{N}(0,\mathrm{Id}_d)$ and $a_n\sim\mu$, setting $I_t^n = \alpha_t z_n + \beta_t a_n$ and $\dot I_t^n = \dot\alpha_t z_n + \dot\beta_t a_n$, the empirical system is
\begin{align}
    \label{eq:empirical:system}
    \hat K_t\,\eta_t &= \hat r_t, \nonumber \\
    [\hat K_t]_{ij} &= \frac{1}{N}\sum_{n=1}^N \nabla\phi_i(I_t^n)\cdot\nabla\phi_j(I_t^n), \nonumber \\
    [\hat r_t]_i &= \frac{1}{N}\sum_{n=1}^N \nabla\phi_i(I_t^n)\cdot\dot I_t^n.
\end{align}
This $P\times P$ system is solved once on a time grid before generation; the coefficients $\{\eta_{t_k}\}$ are reused for any number of samples---no further training is required.

\subsection{Estimation Error and Optimal Diffusion}
\label{sec:estimation:error}

Since $\phi:\R^d\to\R^P$ is finite-dimensional, the kernel $k(x,y) = \phi(x)^\top\phi(y)$ has rank at most $P$ and is not characteristic. The velocity estimate $\hat b_t = \nabla\phi^\top\eta_t$ from Proposition~\ref{th:main} is therefore the best approximation in class, not the exact velocity. The corresponding score estimate, obtained via~\eqref{eq:score:drift}, is
\begin{equation}
    \label{eq:score:from:drift}
    \hat s_t(x) = \frac{\beta_t\,\hat b_t(x) - \dot\beta_t\, x}{\alpha_t\gamma_t}.
\end{equation}
Since this relation is exact, errors in $\hat b_t$ translate to proportional errors in $\hat s_t$. Using~\eqref{eq:score:from:drift} to express the score in terms of the drift, the SDE~\eqref{eq:sde} becomes
{\footnotesize
\begin{equation}
    \label{eq:sde:explicit}
    dX_t = \left[\left(1 + \frac{D_t\beta_t}{\alpha_t\gamma_t}\right)\nabla\phi(X_t)^\top \eta_t - \frac{D_t\dot\beta_t}{\alpha_t\gamma_t}\,X_t\right]dt + \sqrt{2D_t}\,dW_t.
\end{equation}
}
The generated density $\rho_t^D$ depends on $D_t$. The choice of $D_t$ matters precisely because the estimates are inexact: if $\hat b_t = b_t$ and $\hat s_t = s_t$, the law of $X_t$ would equal that of $I_t$ for any $D_t \ge 0$~\citep{albergo2023stochastic}. By Girsanov's theorem, the KL divergence between the path measures of the exact and approximate processes is
{\footnotesize
\begin{equation}
    \label{eq:path:kl}
    D_{\mathrm{KL}}^{\mathrm{path}} = \frac{1}{4}\int_0^1 \frac{1}{D_t}\left(1 + \frac{D_t\,\beta_t}{\alpha_t\gamma_t}\right)^{\!2}\,\E\big[|b_t(I_t) - \nabla\phi(I_t)^\top \eta_t|^2\big]\,dt.
\end{equation}
}
By the data processing inequality, this bounds the KL divergence between the generated and target distributions at the terminal time:
\begin{equation}
    \label{eq:terminal:kl}
    D_{\mathrm{KL}}(\mu_1^D \| \mu) \le D_{\mathrm{KL}}^{\mathrm{path}},
\end{equation}
where $\mu_1^D$ is the distribution of $X_1$ and $\mu$ is the target distribution.
Minimizing the path KL over $D_t$ therefore controls the generation error.

\begin{proposition}[Optimal diffusion coefficient {\citep{ma_sit_2024,chen2024probabilistic}}]
    \label{prop:optimal:eps}
    The pointwise minimizer of the integrand in~\eqref{eq:path:kl} over $D_t > 0$ is
    \begin{equation}
        \label{eq:optimal:eps}
        D_t^* = \frac{\alpha_t\gamma_t}{\beta_t} = \frac{\alpha_t(\alpha_t\dot\beta_t - \dot\alpha_t\beta_t)}{\beta_t}\,.
    \end{equation}
\end{proposition}

\begin{proof}
    Writing $\lambda_t = 1 + D_t\beta_t/(\alpha_t\gamma_t)$, the integrand is $\lambda_t^2D_t^{-1} = D_t^{-1} + 2\beta_t/(\alpha_t\gamma_t) + D_t\beta_t^2/(\alpha_t\gamma_t)^2$. Differentiating and setting to zero gives $D_t^* = \alpha_t\gamma_t/\beta_t$.
\end{proof}

The optimal $D_t^*$ is large for small $t$---providing strong diffusion when the distribution is nearly Gaussian---and vanishes as $t\to 1$, where the SDE reduces to the probability flow ODE; both limits are handled seamlessly by the integrator in Section~\ref{sec:integrator}. $D_t^*$ is also the unique choice for which the time-reversed SDE becomes score-free (Appendix~\ref{app:time:reversal}). The alternative $D\to\infty$, natural in~\citet{lempereur2025mgd}, is not appropriate here (Appendix~\ref{app:exponential}).

\subsection{Optimal Diffusion and Integrator}
\label{sec:integrator}

For the optimal $D_t^* = \alpha_t\gamma_t/\beta_t$, we have $1 + D_t^*\beta_t/(\alpha_t\gamma_t) = 2$ and $D_t^*\dot\beta_t/(\alpha_t\gamma_t) = \dot\beta_t/\beta_t$, so~\eqref{eq:sde:explicit} reduces to
\begin{equation}
    \label{eq:sde:explicit:2}
    dX_t = \Big[2\nabla\phi(X_t)^\top \eta_t - \frac{\dot\beta_t}{\beta_t}\,X_t\Big]dt + \sqrt{\frac{2\alpha_t\gamma_t}{\beta_t}}\,dW_t.
\end{equation}
This can be written as $d(\beta_t X_t) = 2\beta_t\nabla\phi(X_t)^\top \eta_t\, dt + \sqrt{2\alpha_t\beta_t\gamma_t}\,dW_t$. Integrating from $t$ to $t+h$, treating $\nabla\phi(X_s)^\top\eta_s$ explicitly at time $t$, and approximating the integrals by the trapezoidal rule, dividing by $\beta_{t+h}$ gives the integrator:
\begin{equation}
    \label{eq:exp:integrator}
    \begin{split}
        X_{t+h} ={}& \frac{\beta_t}{\beta_{t+h}} X_t + h\!\left(1+\frac{\beta_t}{\beta_{t+h}}\right)\! \nabla\phi(X_t)^\top\eta_t \\
        &+ \frac{\sqrt{h(\alpha_t\beta_t\gamma_t + \alpha_{t+h}\beta_{t+h}\gamma_{t+h})}}{\beta_{t+h}}\,g_t,
    \end{split}
\end{equation}
where $g_t\sim\mathsf{N}(0,\mathrm{Id})$. This scheme handles the singular endpoint $t=0$ seamlessly: as $t\to 0$ (where $D_t^*\to\infty$ and $\beta_t\to 0$), the contribution from $X_t$ vanishes and $X_{t+h}$ is resampled from a Gaussian. No clamping of $D_t$ is required.

The generation procedure is summarized in Algorithm~\ref{alg:generation}.
\vspace{-10pt}
\begin{algorithm}
   \caption{Kernelized Stochastic Interpolant Generation (with optimal $D_t^*$)}
   \label{alg:generation}
\begin{algorithmic}[1]
   \STATE {\bfseries Input:} Feature gradients $\nabla\phi_i$, $i=1,\ldots,P$; data pairs $(z_n,a_n)_{n=1}^N$; steps $K$; schedule $\alpha_t,\beta_t,\gamma_t$
   \STATE {\bfseries Precompute:} For each $t_k = k/K$, $k=0,\ldots,K{-}1$: solve~\eqref{eq:empirical:system} for $\eta_{t_k}\in\R^P$
   \STATE Sample $X_0\sim\mathsf{N}(0,\mathrm{Id}_d)$
   \FOR{$k=0$ {\bfseries to} $K{-}1$}
      \STATE Set $t=t_k$, $h=1/K$
      \STATE $X_{t+h} \leftarrow \dfrac{\beta_t}{\beta_{t+h}} X_t + h\!\left(1+\dfrac{\beta_t}{\beta_{t+h}}\right)\!\nabla\phi(X_t)^\top\eta_t + \dfrac{\sqrt{h(\alpha_t\beta_t\gamma_t + \alpha_{t+h}\beta_{t+h}\gamma_{t+h})}}{\beta_{t+h}}\,g_t$,\\ \hspace{2em} where $g_t\sim\mathsf{N}(0,\mathrm{Id}_d)$
   \ENDFOR
   \STATE {\bfseries Output:} $X_1 \approx a\sim\mu$
\end{algorithmic}
\end{algorithm}

\vspace{-20pt}
\subsection{Choice of Feature Map}
\label{sec:practical}

The choice of feature map $\phi$ determines how well $\hat b_t = \nabla\phi^\top\eta_t$ approximates the true velocity field. We discuss two natural constructions.
\vspace{-10pt}
\paragraph{Scattering Spectra.}
For time series ($d=T$) or random fields ($d = T_1\times T_2$), the wavelet scattering transform~\citep{mallat2012group} provides a natural feature map. The scattering coefficients $S[p]x$ at various paths $p$ give a finite-dimensional, Lipschitz-continuous, translation-invariant representation. Setting $\phi_p(x) = S[p]x$ and computing $\nabla\phi_p(x) = \nabla_x S[p]x$ via backpropagation provides the feature gradients. This choice is well-suited to processes with multiscale structure, such as textures and turbulence~\citep{bruna2019multiscale}.
\vspace{-6pt}
\paragraph{Pretrained Generative Models.}
If $b_t^i(x)$, $i=1,\ldots,P$ are pretrained velocity fields (from flow matching, stochastic interpolants, etc.), we can use them directly as feature gradients by setting $\nabla\phi_i(x) = b_t^i(x)$. The $P\times P$ linear system becomes
{
\small
\begin{equation}
    \label{eq:pretrained:system}
    \sum_{j=1}^P \frac{1}{N}\sum_{n=1}^N \big(b_t^i(I_t^n)^\top b_t^j(I_t^n)\big)\;\eta_t^j = \frac{1}{N}\sum_{n=1}^N b_t^i(I_t^n)^\top\dot I_t^n,
\end{equation}
}
and the combined drift is $\hat b_t(x) = \sum_{i=1}^P \eta_t^i\,b_t^i(x)$. This enables combination of models with different architectures, training stages, or data domains, without any retraining.

\section{Numerical Illustrations}
\label{sec:experiments}

\subsection{Multiscale Time Series: Financial Log-Returns}
\label{sec:1d}

We consider daily log-returns of the S\&P~500 from January 2000 to February 2024. Financial log-returns exhibit heavy tails, volatility clustering, and the leverage effect, making this a challenging test case. Figure~\ref{fig:sampling_1D} shows that Algorithm~\ref{alg:generation} with $K = 1{,}200$ steps (trigonometric schedule, $P=217$ scattering features) faithfully reproduces the return density including heavy tails, and correctly captures the leverage effect~\citep{morel2024path}.

\begin{figure}[t]
    \centering
    \includegraphics[width=\linewidth]{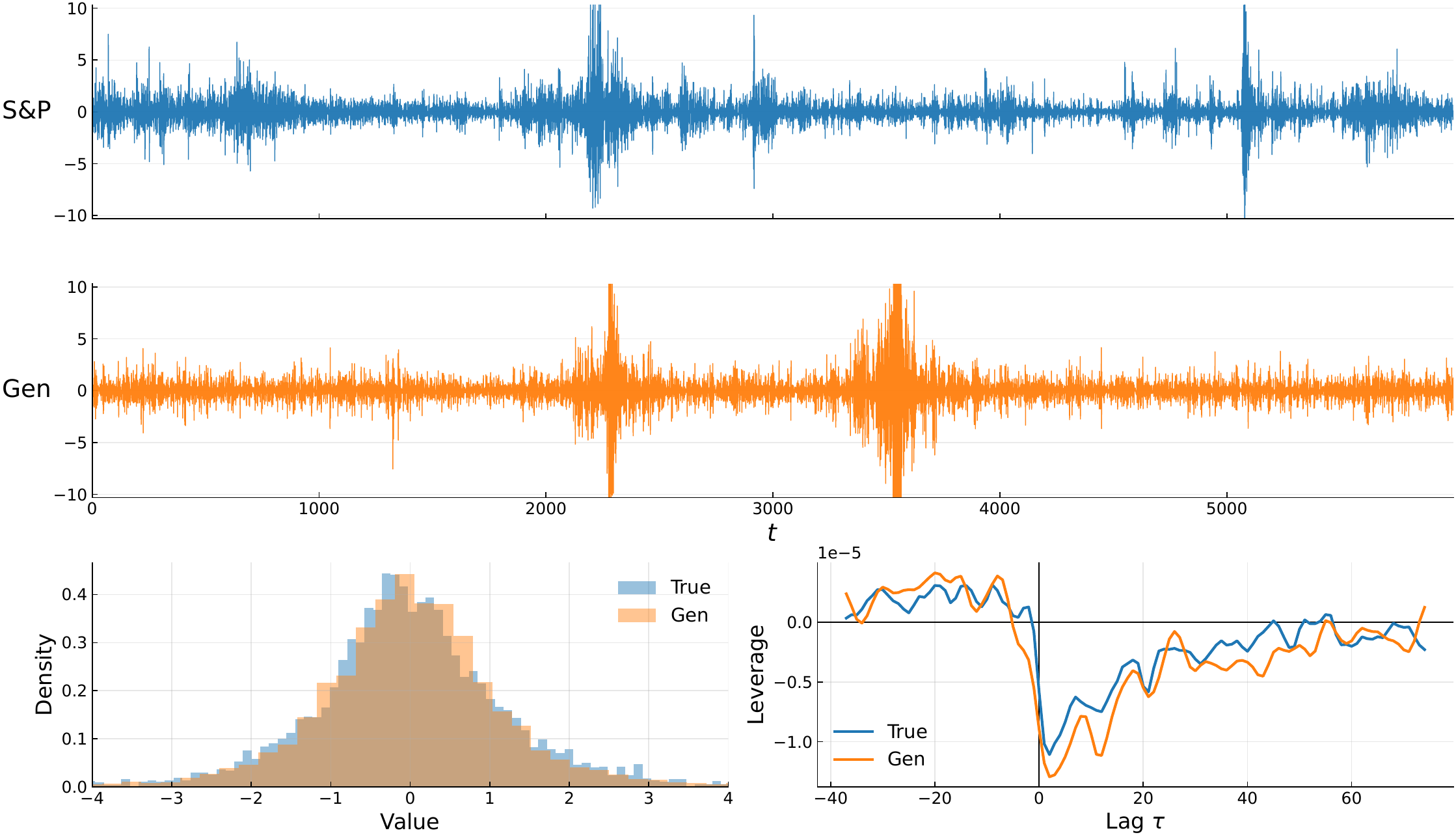}
    \caption{Generation from a single S\&P~500 daily log-returns realization ($d = 6{,}064$) using scattering features ($P=217$). \textbf{Top}: Original (blue) and generated sample (orange), both showing volatility clustering. \textbf{Bottom left}: Log-return densities, with near-perfect agreement including heavy tails. \textbf{Bottom right}: Leverage effect (correlation between past returns and current volatility).}
    \label{fig:sampling_1D}
    \vspace{-15pt}
\end{figure}

\subsection{Physical Fields}
\label{sec:physics:exp}

We test the framework on two-dimensional physical fields using Morlet wavelet scattering features~\citep{cheng2024scattering}.
\vspace{-18pt}
\paragraph{Datasets.} We consider four $64\times64$ fields: \textit{3D turbulence} pressure slices~\citep{li2008public,Perlman_2007}, \textit{dark matter} log-density slices~\citep{villaescusa2020quijote}, \textit{magnetic turbulence} MHD vorticity~\citep{allys2019rwst}, and \textit{weak lensing} convergence maps~\citep{gupta2018non}. See Appendix~\ref{app:details} for details.
\vspace{-10pt}

\paragraph{Generation.} Figure~\ref{fig:fields_generation} shows ground-truth (top) and generated (bottom) samples using Algorithm~\ref{alg:generation} with $K=5{,}000$ steps and the optimal diffusion schedule. The generated fields faithfully reproduce the visual characteristics of each dataset: intermittent turbulent bursts, the filamentary cosmic web, coherent MHD vortices, and sparse lensing spikes.

\begin{figure}[t!]
    \centering
    \includegraphics[width=\linewidth]{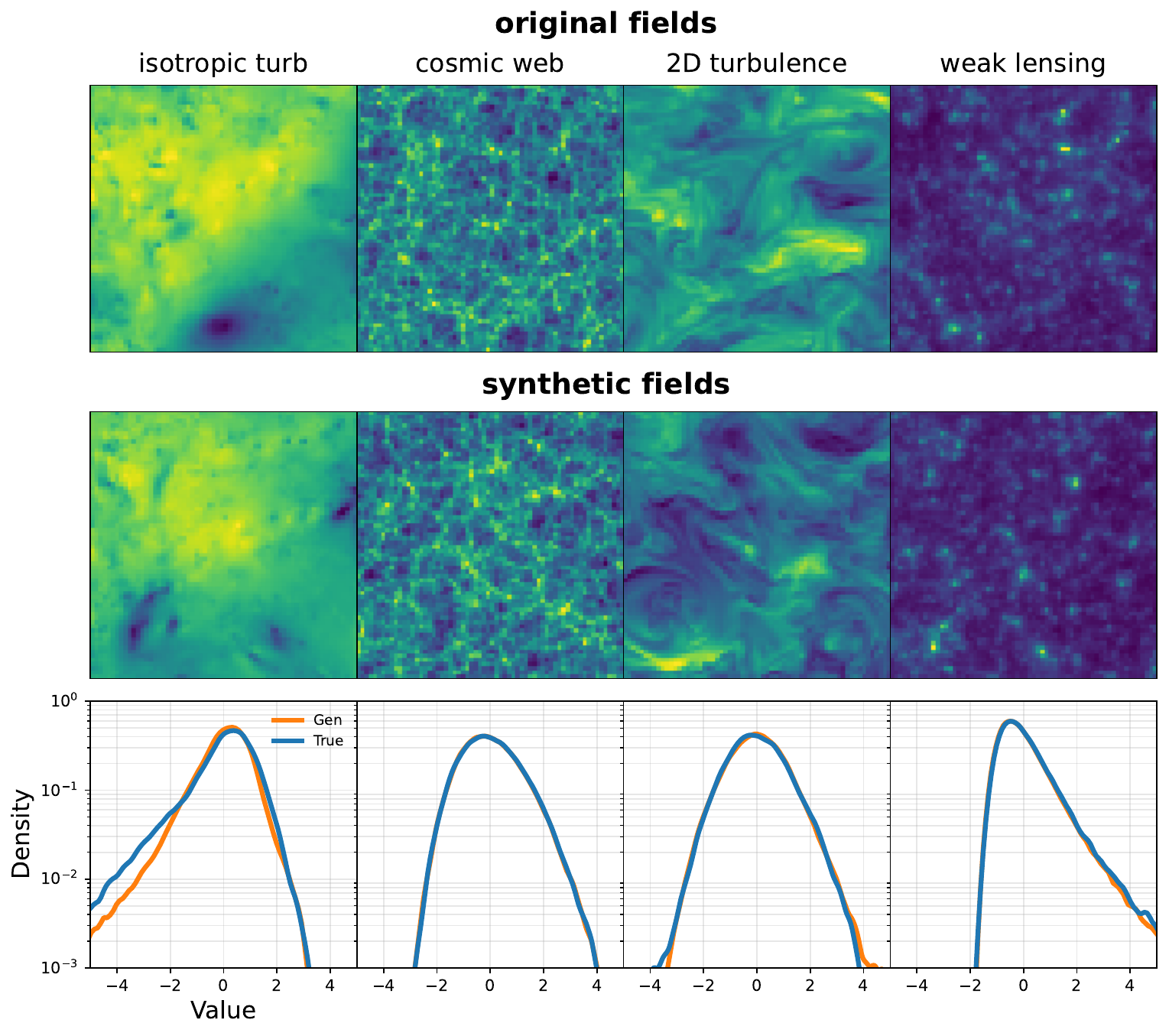}
    \caption{Generation of two-dimensional physical fields using scattering-transform features. \textbf{Top row}: ground-truth samples. \textbf{Bottom row}: samples generated by Algorithm~\ref{alg:generation} with $K=5{,}000$ steps. From left to right: 3D turbulence (pressure), dark matter (log-density), 3D magnetic turbulence (vorticity), and weak lensing (convergence).}
    \label{fig:fields_generation}
    \vspace{-6pt}
\end{figure}

\subsection{Ensemble of Weak Generative Models}
\label{sec:weak}

We evaluate the pretrained-model combination strategy on MNIST using two cohorts of 20 U-Net models trained for only 50 or 100 SGD steps (mini-batches of 128, linear schedule $\alpha_t=1-t$, $\beta_t=t$)---far too few for individual models to produce coherent digits (Figure~\ref{fig:mnist_ensemble}, top two rows).

Applying Algorithm~\ref{alg:generation} with $N=10{,}000$ pairs and $K=1{,}000$ steps, the kernelized ensemble of 20 models (100-step cohort) produces recognizable digits without additional training. To quantify, we compute oracle log-likelihood of $10{,}000$ generated samples under a fully trained U-Net: Figure~\ref{fig:mnist_ensemble} (right) shows it improves monotonically with $P$ and saturates near $P\approx 15$.

\begin{figure}[t]
    \centering
    \includegraphics[width=\linewidth]{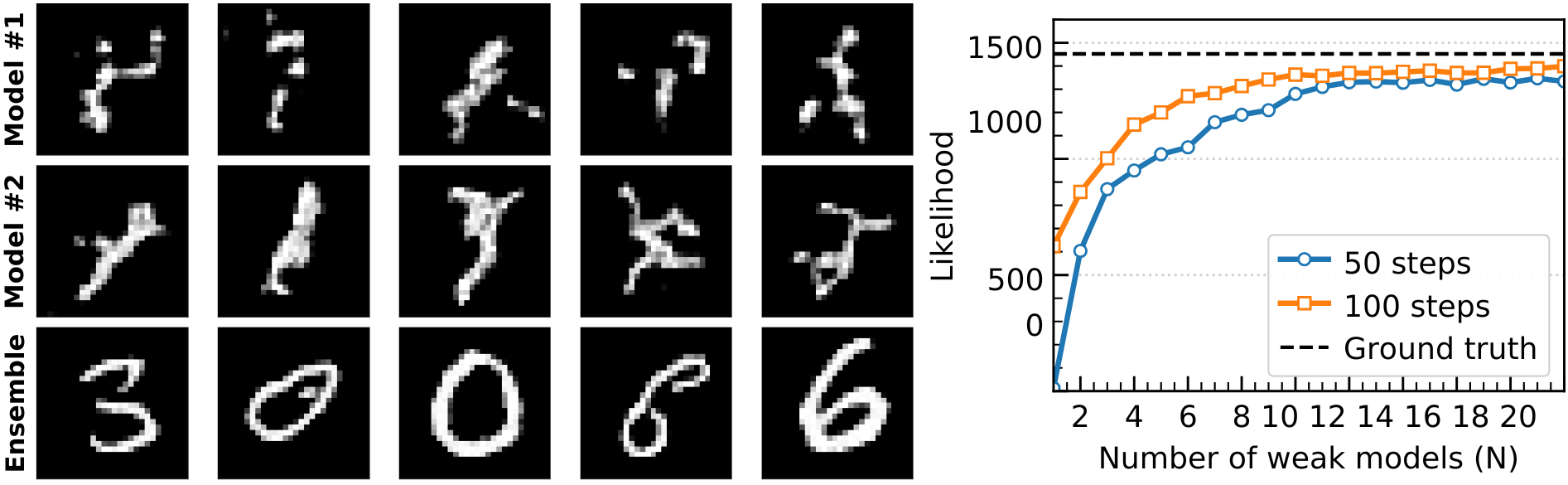}
    \caption{\textbf{Left}: MNIST samples. Rows 1--2: individual models (50 and 100 SGD steps). Row 3: kernelized ensemble ($P=20$, 100-step cohort). \textbf{Right}: Oracle log-likelihood vs.\ ensemble size $P$. Blue: 50-step; orange: 100-step. Error bars: $\pm 1$ std over 5 subsets.}
    \label{fig:mnist_ensemble}
    \vspace{-15pt}
\end{figure}

\section{Conclusion}
\label{sec:conclusion}

We have reformulated stochastic interpolant drift estimation as a $P\times P$ linear system over feature gradients, with $P$ independent of the data dimension. The optimal diffusion coefficient $D_t^*$, derived via Girsanov's theorem, minimizes a KL bound on generation error; the resulting integrator handles its singularity at $t=0$ without clamping. Feature maps from scattering transforms and pretrained velocity fields enable generation and model combination without neural network training. This approach is complementary to Moment-Guided Diffusion~\citep{lempereur2025mgd}; combining the two methods is a natural direction for future work.


\bibliography{refs}
\bibliographystyle{icml2026}

\newpage
\onecolumn
\appendix

\section{General Hilbert Space Formulation}
\label{app:hilbert}

The finite-dimensional framework of Section~\ref{sec:theory} extends to general Hilbert spaces. Let $\phi:\R^d\to\mathcal{F}$ be a feature map into a (possibly infinite-dimensional) Hilbert space $\mathcal{F}$ with inner product $\langle\cdot,\cdot\rangle_\mathcal{F}$, and let $k(x,y) = \langle\phi(x),\phi(y)\rangle_\mathcal{F}$ be the associated positive-definite kernel.

The drift ansatz becomes $\hat b_t(x) = \langle\nabla\phi(x),\eta_t\rangle_\mathcal{F}$, where $\nabla\phi(x)$ maps from $\R^d$ to $\mathcal{F}$ and $\langle\nabla\phi(x),\eta_t\rangle_\mathcal{F}\in\R^d$. The operator $K_t = \E[\nabla\phi(I_t)\cdot\nabla\phi(I_t)] \in \mathcal{F}\otimes\mathcal{F}$ generalizes the matrix~\eqref{eq:Kt}, and the linear system $K_t\eta_t = \E[\langle\nabla\phi(I_t),\dot I_t\rangle]$ generalizes~\eqref{eq:bt:kernel}.

\begin{theorem}[Exact recovery under characteristic kernels]
    \label{th:characteristic}
    If $k$ is characteristic and $K_t$ is positive-definite for all $t\in[0,1]$, then $\hat b_t(x) = \langle\nabla\phi(x),\eta_t\rangle_\mathcal{F}$ recovers the true velocity field $b_t(x) = \E[\dot I_t|I_t=x]$ exactly for all $t$.
\end{theorem}

\begin{proof}
    Consider the ODE $\dot X_t^* = \langle\nabla\phi(X_t^*),\eta_t^*\rangle_\mathcal{F}$ with $X_0^*=I_0\sim\mathsf{N}(0,\mathrm{Id}_d)$, where $\eta_t^*$ solves
    $\E_{X_t^*}[\nabla\phi(X_t^*)\cdot\nabla\phi(X_t^*)]\,\eta_t^* = \E_{I_t}[\nabla\phi(I_t)\cdot \dot I_t]$.
    The MMD between $I_t$ and $X_t^*$ satisfies
    \begin{equation}
        \frac{d}{dt}\|\E[\phi(I_t)] - \E[\phi(X_t^*)]\|_\mathcal{F}^2 = 0
    \end{equation}
    by the chain rule and the definition of $\eta_t^*$. Since $X_0^* = I_0$, the MMD vanishes for all $t$. Characteristicness implies $X_t^*\overset{d}{=}I_t$, so the two linear systems coincide and $\hat b_t = b_t$.
\end{proof}

Examples of characteristic kernels include the Gaussian RBF $k(x,y)=\exp(-\|x-y\|^2/2\sigma^2)$, the Laplacian $k(x,y)=\exp(-\gamma\|x-y\|_1)$, and inverse multiquadratic kernels. All require infinite-dimensional feature spaces.

\section{Time Reversal of the Optimal SDE}
\label{app:time:reversal}

We show that the time reversal of the optimal SDE~\eqref{eq:sde:explicit:2} is an Ornstein--Uhlenbeck-type process, independent of the target distribution $\mu$.

\paragraph{Forward process.} With the exact drift $b_t$ and score $s_t$, the optimal SDE~\eqref{eq:sde} with $D_t^* = \alpha_t\gamma_t/\beta_t$ can be written as
\begin{equation}
    \label{eq:forward:app}
    dX_t = \left[\gamma_t\, s_t(X_t) + \frac{\dot\beta_t}{\beta_t}\,X_t\right]dt + \sqrt{\gamma_t}\,dW_t.
\end{equation}
To see this, note that from the score-drift relation~\eqref{eq:score:drift}, $b_t = \frac{\alpha_t\gamma_t}{\beta_t}s_t + \frac{\dot\beta_t}{\beta_t}x$, so
\begin{equation}
    b_t + D_t^* s_t = 2b_t - \frac{\dot\beta_t}{\beta_t}x = \frac{2\alpha_t\gamma_t}{\beta_t}s_t + \frac{\dot\beta_t}{\beta_t}x = \gamma_t\, s_t + \frac{\dot\beta_t}{\beta_t}x.
\end{equation}

\paragraph{Time reversal.} Let $\tau = 1-t$ and $Y_\tau = X_{1-\tau}$. By the standard time-reversal formula for diffusions, $Y_\tau$ satisfies
\begin{equation}
    dY_\tau = \left[-f_{1-\tau}(Y_\tau) + \gamma_{1-\tau}\,s_{1-\tau}(Y_\tau)\right]d\tau + \sqrt{\gamma_{1-\tau}}\,d\bar W_\tau,
\end{equation}
where $f_t(x) = \gamma_t s_t(x) + \frac{\dot\beta_t}{\beta_t}x$ is the forward drift. Substituting:
\begin{equation}
    \begin{aligned}
        -f_{1-\tau}(Y_\tau) + \gamma_{1-\tau}\,s_{1-\tau}(Y_\tau)
        &= -\gamma_{1-\tau}\,s_{1-\tau}(Y_\tau) - \frac{\dot\beta_{1-\tau}}{\beta_{1-\tau}}\,Y_\tau + \gamma_{1-\tau}\,s_{1-\tau}(Y_\tau)\\
        &= -\frac{\dot\beta_{1-\tau}}{\beta_{1-\tau}}\,Y_\tau.
    \end{aligned}
\end{equation}
The score terms cancel, leaving
\begin{equation}
    \label{eq:reversed:sde}
    dY_\tau = -\frac{\dot\beta_{1-\tau}}{\beta_{1-\tau}}\,Y_\tau\,d\tau + \sqrt{\frac{2\alpha_{1-\tau}\gamma_{1-\tau}}{\beta_{1-\tau}}}\,d\bar W_\tau.
\end{equation}
This is a linear SDE with no dependence on the target distribution $\mu$ or the score $s_t$---it is purely determined by the interpolant schedule $\alpha_t$, $\beta_t$.

\paragraph{Interpretation.} The cancellation of the score is a consequence of the optimal choice $D_t^* = \alpha_t\gamma_t/\beta_t$. For other choices of $D_t$, the reversed process would still depend on $s_t$ and hence on the target $\mu$. This is the mirror image of standard score-based diffusion~\citep{song2020score}, where the forward process is an OU process (score-free) and the backward process requires the score; here, the forward process uses the score, and with $D_t^*$, the backward process reduces to OU-type dynamics.

\section{Exponential Families and the $D\to\infty$ Limit}
\label{app:exponential}

For finite $D_t$, the generated density $\rho_t^D$ is not in any particular parametric family. However, in the limit $D\to\infty$ (constant in time), the Fokker--Planck equation
\begin{equation}
    \partial_t\rho_t^D + \nabla\cdot\!\big(\hat b_t\;\rho_t^D\big) = D\,\nabla\cdot\!\big(\hat s_t\;\rho_t^D + \nabla\rho_t^D\big)
\end{equation}
becomes dominated by its right-hand side, which equilibrates to an exponential family member $\bar\rho_t \propto \exp(-\theta_t^\top\Phi)$ for some $\theta_t$ determined by $\hat s_t$, provided that the corresponding SDE has an invariant distribution. Even if this distribution exists, it is a maximum entropy distribution matching the moments implied by the estimated drift $\hat b_t$, rather than the true moments $\E[\Phi(I_t)]$. There is no reason to expect this to be closer to the target than what is achieved with finite $D_t$.

This explains why the $D\to\infty$ limit is natural in Moment-Guided Diffusion~\citep{lempereur2025mgd}, where the target moments are imposed explicitly, but not here: finite $D_t^*$ better exploits the information in $\hat b_t$ by preserving the transport dynamics.

\section{Experimental Details}
\label{app:details}

\textbf{MNIST U-Net.} A sinusoidal time embedding (scaled for $t\in[0,1]$) is passed through a small MLP and injected into every UNetBlock by projection to channel size and additive conditioning. The U-Net has two encoder stages ($32\to64$ channels) with $3\times3$ convolutions, BatchNorm, and ReLU, separated by $2\times2$ max pooling, followed by a 128-channel bottleneck. The decoder mirrors this with transposed-conv upsampling and skip connections. A final $1\times1$ conv produces the single-channel output.

\paragraph{Scattering Spectra.}
The scattering spectra model exploits cross-scale dependencies by considering the covariances between wavelet transform moduli of the field, compressed in a wavelet basis. A field of size $d$ is described with $O(\log(d)^3)$ low-degree functions that capture moments of multiscale processes up to order four.

For the S\&P time series, we follow~\citet{morel2024path} and use $J = 8$ dyadic scales, leading to $P=217$ coefficients. For two-dimensional physical fields, we follow~\citet{cheng2024scattering} and use $J=6$ dyadic scales and $L=4$ orientations, yielding $P=6803$ coefficients. In both cases, the scattering transform is computed with the Morlet wavelet~\citep{morlet1982wavea,morlet1982waveb}. Feature gradients $\nabla\phi_p(x)$ are computed via automatic differentiation.

\section{Cross-Domain Model Composition}
\label{app:cross:domain}

A distinctive feature of our framework is that the pretrained velocity fields in the ensemble need not be trained on the same dataset. Since the linear system~\eqref{eq:pretrained:system} is solved using data pairs $(z_n, a_n)$ from the \emph{target} distribution, the coefficients $\eta_t$ optimally reweight the feature gradients for the target, regardless of their training provenance.

\paragraph{Setup.} We train 10 weak models (50 SGD steps, mini-batches of 128) on each of three source domains---Fashion-MNIST, EMNIST (letters), and Kuzushiji-MNIST---as well as 10 models on the target domain, MNIST. Each model uses the same U-Net backbone with an independent random initialization. At inference, we compose all 40 models ($P=40$) and solve the linear system~\eqref{eq:pretrained:system} using $N=10{,}000$ MNIST data pairs. The 10 MNIST models alone serve as a baseline.

\paragraph{Results.} Figure~\ref{fig:adatpation_enhancement} displays representative samples. The top four rows show outputs from individual weak models trained on each domain. The bottom row shows samples generated by composing all 40 models: the resulting digits are sharper and more coherent than those from the 10 MNIST models alone. The source-domain velocity fields provide useful low-level feature gradients (edges, strokes, curves) that the linear system repurposes for MNIST generation.

\begin{figure}[t]
    \centering
    \includegraphics[width=\linewidth]{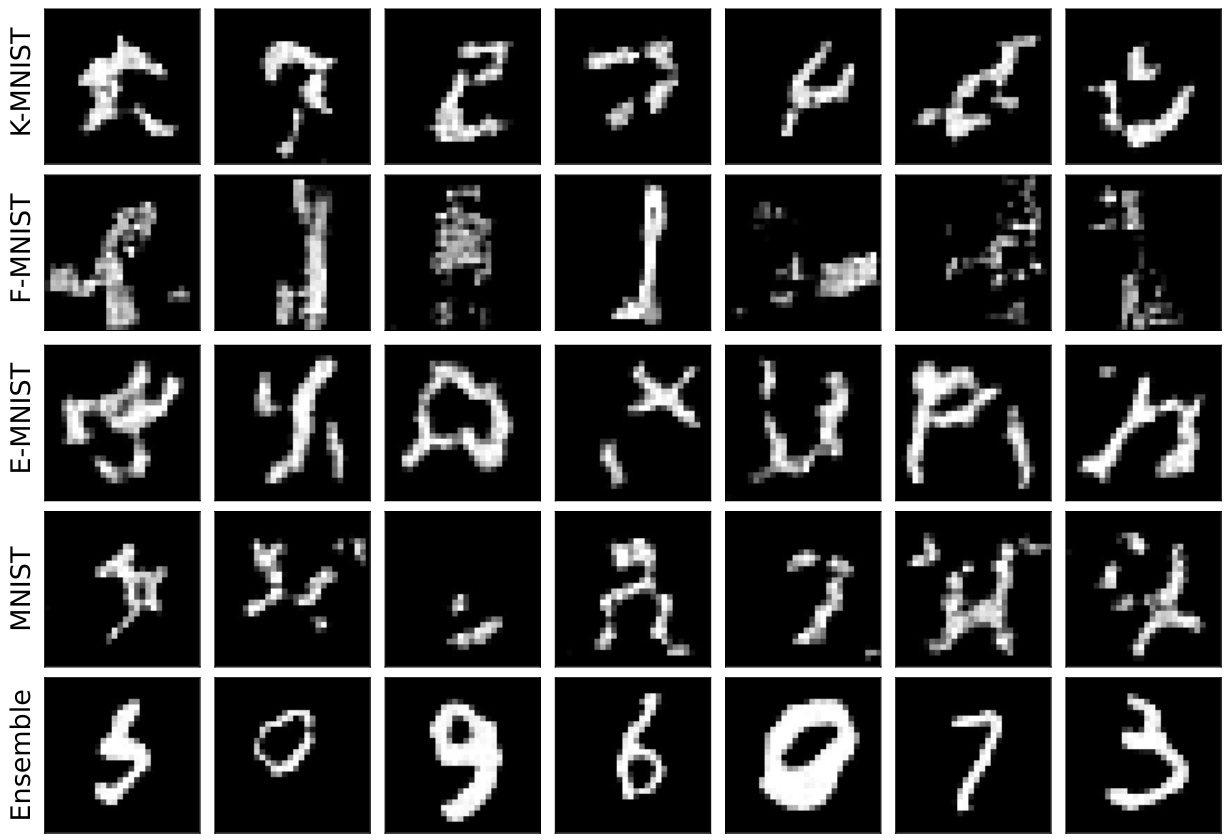}
    \caption{Cross-domain model composition for MNIST generation. Rows 1--4: samples from individual weak models trained on Kuzushiji-MNIST, Fashion-MNIST, EMNIST (letters), and MNIST, respectively. Row 5: samples generated by composing all 40 source- and target-domain models via Algorithm~\ref{alg:generation}, using MNIST data to solve the linear system.}
    \label{fig:adatpation_enhancement}
\end{figure}

\end{document}

%% file: math_commands.tex
\usepackage{amsmath,amsfonts,bm}

\def\eqref#1{equation~\ref{#1}}

\def\1{\bm{1}}

\DeclareMathAlphabet{\mathsfit}{\encodingdefault}{\sfdefault}{m}{sl}
\SetMathAlphabet{\mathsfit}{bold}{\encodingdefault}{\sfdefault}{bx}{n}

\newcommand{\E}{\mathbb{E}}

\newcommand{\R}{\mathbb{R}}

